\documentclass{article}

\usepackage[preprint, nonatbib]{neurips_2020}

\usepackage[utf8]{inputenc} % allow utf-8 input
\usepackage[T1]{fontenc}    % use 8-bit T1 fonts
\usepackage{hyperref}       % hyperlinks
\usepackage{url}            % simple URL typesetting
\usepackage{booktabs}       % professional-quality tables
\usepackage{amsfonts}       % blackboard math symbols
\usepackage{nicefrac}       % compact symbols for 1/2, etc.
\usepackage{microtype}      % microtypography

\usepackage{amsmath}
\usepackage{mathtools}
\usepackage{amssymb}
\usepackage[shortlabels]{enumitem}
\usepackage{graphicx}
\usepackage{color}
\usepackage{subfig}
\usepackage{schemata}
\usepackage[ruled,vlined]{algorithm2e}
\makeatletter
\newcommand{\algorithmstyle}[1]{\renewcommand{\algocf@style}{#1}}
\makeatother
\usepackage{adjustbox}
\usepackage{float}
\usepackage{amsthm}
\usepackage{chngcntr}
\usepackage{apptools}
\usepackage{thmtools}
\usepackage{thm-restate}
\usepackage{cleveref}
\usepackage{wrapfig}
\usepackage{enumitem}

\newtheorem{corollary}{Corollary}
\newtheorem{assumption}{Assumption}

\title{Optimistic Distributionally Robust Policy Optimization}
\author{%
  Jun Song and Chaoyue Zhao \\
  Department of Industrial Engineering\\
  University of Washington\\
  \texttt{\{juns113, cyzhao\}@uw.edu} \\
}

\begin{document}
\maketitle

\begin{abstract}
 Trust Region Policy Optimization (TRPO) and Proximal Policy Optimization (PPO), as the widely employed policy based reinforcement learning (RL) methods, are prone to converge to a sub-optimal solution as they limit the policy representation to a particular parametric distribution class. To address this issue, we develop an innovative Optimistic Distributionally Robust Policy Optimization (\texttt{ODRPO}) algorithm, which effectively utilizes Optimistic Distributionally Robust Optimization (DRO) approach to solve the trust region constrained  optimization problem without parameterizing the policies. Our algorithm improves TRPO and PPO with a higher sample efficiency and a better performance of the final policy while attaining the learning stability. Moreover, it achieves a globally optimal policy update that is not promised in the prevailing policy based RL algorithms. Experiments across tabular domains and robotic locomotion tasks demonstrate the effectiveness of our approach. 
\end{abstract}

\section{Introduction}
Model-free reinforcement learning (RL) has been extensively studied and successfully implemented in different domains, notably in video games \cite{mnih2013_dqn, mnih2015_dqn}, board games \cite{silver2016_go_game, heinrich2016_poker}, robotics \cite{guholly2017_robotics, grudic2003_pg_robotics, peters2006_pg_robotics}, and continuous control tasks \cite{duan2016_continuous, lillicrap2015_ddpg, schulmanl2016_gae_continuous}.  One prominent policy based RL approach is Policy Gradient (PG) \cite{grudic2003_pg_robotics, peters2006_pg_robotics, lillicrap2015_ddpg, sutton1999_pg, williams1992_reinforce, mnih2016_a3c, silver2014_dpg}, and its basic idea is to represent the policy with a parametric probability distribution $\pi_\theta (a|s) = P [a|s; \theta]$, such that the action $a$ in state $s$ is chosen stochastically following the policy $\pi_\theta$ controlled by parameter $\theta$. However, the determination of the step size for the policy update remains challenging in order to balance the tradeoff between learning stability and learning speed. 

Trust Region Policy Optimization (TRPO) \cite{schulman2015_trpo} addresses this issue by imposing a Kullback-Leibler (KL) divergence based constraint (trust region constraint) to restrict the size of the policy update. To approximately solve the trust region constrained optimization problem iteratively, a second order optimizer Conjugate Gradient \cite{hestenes1952_conjugategradient} is used to update policy parameters. This procedure is theoretically justified to guarantee a monotonic improvement of the expected reward. Another way to prevent an excessive shift between the update of the policy is to penalize the update. For example, Proximal Policy Optimization (PPO) \cite{schulman2017_ppo} penalizes on the KL divergence between the old and the new policies, and it can be efficiently solved by a first-order method like Gradient Descent. Similarly, the Behavior Guided Policy Gradient (BGPG) \cite{pacchiano2019_bgpg} considers the entrophy-regularized Wasserstein distance between the old and the new policies, and penalizes the Wasserstein distance to prevent large policy updates.

One limitation for TRPO and its variants such as PPO and BGPG is that they limit the policy representation to a particular parametric distribution class, such that 
the Gaussian \cite{schulman2015_trpo, schulman2017_ppo}, Beta \cite{chou2017_beta} and Delta \cite{lillicrap2015_ddpg, silver2014_dpg} distribution functions. As indicated in \cite{tessler2019_dpo}, since parametric distributions are not convex in the distribution space, optimizing over such distributions will result in local movements in the action space and thus converge to a sub-optimal solution. To address this issue, we propose an Optimistic Distributionally Robust Optimization approach to solve the trust-region constrained optimization problem in TRPO.

Distributionally Robust Optimization (DRO) is a modeling paradigm, where the objective is to find a decision $x$ that minimizes the expected cost under the most adversarial probability
distribution, i.e., $\min_{x \in \mathcal{X}} \max_{\mathbb{P} \in \mathcal{D}} \mathbb{E}_{\mathbb{P}}[Q(x,\xi)]$, where the distribution $\mathbb{P}$ of the random parameter $\xi$ is not precisely known but assumed to belong to an ambiguity set $\mathcal{D}$. Multiple types of ambiguity set $\mathcal{D}$ have been studied recently, such as moment based confidence set \cite{delage2010_dro_moment, zymler2011_dro_moment, goh2010_dro_moment}, $\phi$ divergences based confidence set \cite{hu2012_dro_kl, bental2013_dro_phi} and Wasserstein metric based confidence set \cite{kuhn2015_dro_wasserstein, zhao2018_dro_wasserstein}. Unlike the traditional gradient-based optimization methods, DRO approach releases the assumption that the distribution $\mathbb{P}$ must fall into a particular parametric probability distribution class. Instead, it considers all distributions  that are admissible.

In this paper, we utilize Optimistic DRO approach to tackle the trust region constrained optimization problem. That is, instead of considering the most adversarial probability distribution, we aim to find the most optimistic (i.e., optimal) one to maximize the total reward function. To explore the impact of different metrics on the system performance, we consider both KL divergence and Wasserstein metric to construct the trust region. We highlight our contributions as follows:
\begin{enumerate}
    \item We develop an innovative Optimistic Distributionally Robust Policy Optimization (\texttt{ODRPO}) framework that effectively utilizes Optimistic DRO approach to address the trust region constrained optimization problem. The reformulations of  the trust region constrained optimization problem are derived and the closed-form optimal policies are obtained for both KL-divergence and Wasserstein metric cases.
    
    \item We show that the iterative policy updates performed by \texttt{ODRPO} are globally optimal. Also overall, it can guarantee a monotonic performance improvement through the iterations. Compared with the traditional TRPO and PPO, our proposed approach has a better sample efficiency while maintaining the learning stability. 
    
    \item Since the policy learnt by \texttt{ODRPO} is not confined to the scope of parametric functions, this certainly opens up the possibility of converging to a better final policy. In fact, we observed significant improvements in nearly all our test cases.
    
    \item A comprehensive evaluation on several types of environment including tabular domains and robotic locomotion tasks shows the effectiveness of our approach. The implementation of the proposed \texttt{ODRPO} algorithm can be found at \url{https://github.com/kadysongbb/dr-trpo}.
\end{enumerate}

\section{Background}
We consider an infinite-horizon discounted Markov Decision Process (MDP), defined by the tuple $(\mathcal{S},\mathcal{A},P,r,\rho_0,\gamma)$, where $\mathcal{S}$ is the state space, $\mathcal{A}$ is the action space, $P: \mathcal{S} \times \mathcal{A} \times \mathcal{S} \xrightarrow{} \mathbb{R}$ is the transition probability, $r: \mathcal{S} \times \mathcal{A} \xrightarrow{} \mathbb{R}$ is the reward function, $\rho_0: \mathcal{S} \xrightarrow{} \mathbb{R}$ is the distribution of the initial state $s_0$, and $\gamma$ is the discount factor. We define the return of timestep $t$ as the accumulated discounted reward from $t$, $R_t = \sum_{k=0}^{\infty} \gamma^k r(s_{t+k},a_{t+k})$, and the performance of a stochastic policy $\pi$ as $\eta(\pi) = \mathbb{E}_{s_0,a_0,s_1 \dots} [\sum_{t=0}^{\infty} \gamma^t r(s_t,a_t)]$ where $a_t \sim \pi(a_t|s_t)$, $s_{t+1} \sim P(s_{t+1}|s_t,a_t)$. 

As shown in \cite{kakade2002_approximatelyoptimal}, the expected return of a new policy $\pi'$ can be expressed in terms of the advantage over the old policy $\pi$: $\eta(\pi') = \eta(\pi) + \mathbb{E}_{s\sim \rho^{\pi'}, a \sim \pi'} [A^{\pi}(s,a)]$, where $A^{\pi}(s,a) = \mathbb{E}[R_t|s_t = s, a_t = a; \pi] - \mathbb{E}[R_t|s_t = s; \pi]$ represents the advantage function and $\rho^{\pi}$ represents the unnormalized discounted visitation frequencies, i.e., $\rho^{\pi}(s) = \sum_{t=0}^{\infty}\gamma^t P(s_t = s)$.

Trust Region Policy Optimization (TRPO) \cite{schulman2015_trpo} introduces the following function $L_{\pi}(\pi')$ to approximate $\eta(\pi')$, by sampling the states with $\pi$ instead of $\pi'$: 

\begin{equation}
L_{\pi}(\pi') = \eta(\pi) + \mathbb{E}_{s\sim \rho^{\pi}, a \sim \pi'} [A^{\pi}(s,a)].
\label{local_approximation}
\end{equation}

To obtain the performance bound of $\eta(\pi')$, two particular probability metrics are discussed in \cite{schulman2015_trpo}: total variation metric $d_{TV}$ and KL divergence $d_{KL}$. Let $d_{TV}^{\max}(\pi',\pi) = \max_{s} d_{TV} (\pi'(\cdot|s), \pi(\cdot|s))$ and $d_{KL}^{\max}(\pi',\pi) = \max_{s} d_{KL} (\pi'(\cdot|s), \pi(\cdot|s))$, the following performance bounds can be derived \cite{schulman2015_trpo}: 
\begin{restatable}{thm}{} Let $\epsilon = \max_{s,a}|A^{\pi}(s,a)|$ and $c = \frac{4\epsilon\gamma}{(1-\gamma)^2}$, then the following bounds hold: 
\begin{equation*}
\begin{split}
&\eta(\pi') \ge L_{\pi}(\pi') - c(d_{\tiny{\mbox{TV}}}^{\max}(\pi',\pi))^2, \\
&\eta(\pi') \ge L_{\pi}(\pi') - c d_{\tiny{\mbox{KL}}}^{\max}(\pi',\pi).
\end{split}
\end{equation*}
\label{trpo_theorem}
\end{restatable}
Let $M_{KL}(\pi') = L_{\pi}(\pi') - cd_{KL}^{\max}(\pi',\pi)$, which can be considered as a surrogate function to approximate $\eta(\pi')$. Maximizing $M_{KL}(\pi')$ at each iteration can guarantee a non-decreasing improvement in the expected return $\eta(\pi')$. Similarly, we can claim that $M_{TV}(\pi') = L_{\pi}(\pi') - c(d_{TV}^{\max}(\pi',\pi))^2$ is a surrogate function to approximate $\eta(\pi')$ and maximizing $M_{TV}(\pi')$ can also guarantee a monotonic performance improvement of $\eta(\pi')$.

Instead of using a penalty coefficient $c$ to control the update of the policy, TRPO uses a trust region constraint to restrict that the new policy cannot deviate from the old policy too much. In addition, in TRPO, the policy $\pi$ is parameterized as $\pi_{\theta}$ with the parameter vector $\theta$. For notation brevity, we use $\theta$ to represent the policy $\pi_{\theta}$. Then, the new policy $\theta'$ is found in each iteration to maximize \eqref{local_approximation}, or equivalently, the expected value of the advantage function:
\begin{equation}
\begin{split}
& \max_{\theta'}\ \  \mathbb{E}_{s\sim \rho^{\theta}, a \sim \theta'} [ A^{\theta} (s,a)]  \\
& \text{s.t.} \ \ \mathbb{E}_{s\sim \rho^{\theta}} [d_{M} (\theta', \theta)] \le \delta, 
\end{split}
\label{trpo_problem}
\end{equation}
where $d_M$ represents a general probability metric and $\delta$ is the threshold of the distance between the new and the old policies.

\section{ODRPO: A framework to learn distributional policies}
\label{ODRPO_section_has_proofs}
In this section, we first develop the framework of the Optimistic Distributionally Robust Policy Optimization (\texttt{ODRPO}) under a general probability metric. Then we discuss the solution methodology for \texttt{ODRPO} with KL divergence and Wasserstein metric based trust regions respectively. 

\subsection{Problem formulation and ODRPO framework}
\label{ODRPO_framework}
In our \texttt{ODRPO} model, we release the restrictive assumption that a policy has to follow a parametric distribution class by allowing all admissible policies, that is, the expected distance between the new and the old policies should be within a threshold $\delta$, i.e.,
\begin{equation}
\mathbb{E}_{s\sim \rho^{\pi}} [d_{M} (\pi'(\cdot|s), \pi(\cdot|s))] \le \delta.
\label{metricset}
\end{equation}
Since the optimization goal in each policy update is to obtain the maximal expected value of the advantage function, \texttt{ODRPO} focuses on identifying the optimistic policy that falls in the set depicted in \eqref{metricset}. It is completely opposite to the traditional DRO framework, which aims to find the most adversarial distribution among all admissible distributions. Therefore, the \texttt{ODRPO} framework is shown as follows:
\begin{equation}
\begin{split}
& \max_{\pi' \in \mathcal{D}} \hspace{3mm} \mathbb{E}_{s\sim \rho^{\pi}, a \sim \pi'(\cdot|s)} [ A^{\pi} (s,a)] \\
& \text{where} \hspace{3mm} \mathcal{D} = \{\pi' | \mathbb{E}_{s\sim \rho^{\pi}} [d_{M} (\pi'(\cdot|s), \pi(\cdot|s))] \le \delta \}.
\end{split}
\label{odrpo_problem}
\end{equation}
This framework is closely related to the literature on practicing optimism when facing an uncertain environment \cite{bi2005support,srinivas2009gaussian, hanasusanto2017ambiguous, nguyen2019optimistic}. To study the impact of the choice of metrics on the system performance, we explore two metrics in particular: KL divergence and Wasserstein metric. We will investigate the policy update approach for each case and compare the performance of two cases with several benchmark approaches such as TRPO, PPO, and Advantage Actor Critic (A2C) \cite{mnih2016_a2c}. Before describing our main result, we adopt a mild and conventional assumption that generally holds true in most practical cases:

\begin{assumption}
Assume $A^{\pi} (s,a)$ is bounded, i.e., $\sup_{a \in \mathcal{A}}{|A^{\pi} (s,a)|} < \infty$, $\forall s \in \mathcal{S}$.  
\label{bounded_A}
\end{assumption}

\subsection{ODRPO with KL divergence based trust region}
\label{opt_policy_section1}
We first present the optimal policy for the \texttt{ODRPO} where the trust region is constructed using the KL divergence. In this formation, both the state space and the action space can be discrete or continuous. Theorems \ref{prop_dual_problem_kl} and \ref{thm_optimal_policy_kl} show the dual formulation and the optimal policy distribution of the \texttt{ODRPO} for the KL divergence case. The detailed proofs of the two theorems can be found in Appendix A. 

\begin{restatable}{thm}{thmdualkl}
If Assumption \ref{bounded_A} holds, then the KL trust-region constrained optimization problem in \eqref{odrpo_problem} is equivalent to the following problem: 
\begin{equation}
 \min_{\beta \ge 0} \{\beta\delta  + \mathbb{E}_{s\sim \rho^{\pi}} \beta \log \mathbb{E}_{a \sim \pi(\cdot|s)} [e^{A^\pi (s,a)/\beta}]\}.
\label{kl_trco_dual}
\end{equation}
\label{prop_dual_problem_kl}
\end{restatable}

\begin{restatable}{thm}{thmoptpolicykl}
If Assumption \ref{bounded_A} holds and $\beta^*$ is the global optimal solution to \eqref{kl_trco_dual}, then the optimal policy solution to the KL trust-region constrained optimization problem in \eqref{odrpo_problem} is: 
\begin{equation}
    \pi'^*(a|s) = \frac{e^{ A^\pi (s,a)/\beta^*} \pi(a|s) }{\mathbb{E}_{a \sim \pi(\cdot|s)} [e^{ A^\pi (s,a)/\beta^*}]},  \hspace{5mm} \forall s \in \mathcal{S}, a \in \mathcal{A}.
\label{continuous_KL_optimal}
\end{equation}
\label{thm_optimal_policy_kl}
\end{restatable}
We use gradient-based global optimization algorithms \cite{wales1998_basin_hopping, zhan2006_monte_carlo_basin, leary2000_global_optimization} to find the global optimal $\beta^*$ in \eqref{kl_trco_dual}.
The gradient of the objective in \eqref{kl_trco_dual} is derived as below:
\begin{equation*}
\begin{split}
    \frac{\partial l_0}{\partial \beta} = \delta + \mathbb{E}_{s \sim \rho^\pi}\{ \log \mathbb{E}_{a \sim \pi(\cdot|s)} [e^{A^\pi (s,a)/\beta}] - \frac{ \mathbb{E}_{a \sim \pi(\cdot|s)} [e^{A^\pi (s,a)/\beta} \times A^\pi (s,a)]}{\beta \mathbb{E}_{a \sim \pi(\cdot|s)} [e^{A^\pi (s,a)/\beta}]}   \}. 
\end{split}
\end{equation*}
In implementation, we obtain the global optimal $\beta^*$ using the Basin Hopping algorithm \cite{wales1998_basin_hopping}, which is a two-phase global optimization method that utilizes the gradient information to speed up the search in the local phase. 
 
\subsection{ODRPO with Wasserstein metric based trust region}
\label{opt_policy_section2}

In this subsection, we show that the Wasserstein metric based trust region can also guarantee a monotonic improvement given that the action space is finitely discrete, and derive the optimal policy for the \texttt{ODRPO} where the trust region is constructed using the Wasserstein metric.

Given two probability distributions $\mu$ and $v$ on the supporting space, the Wasserstein metric is defined as: 
\begin{equation*}
d_W(\mu,v) = \inf_\Pi \mathbb{E}_\Pi{[d(X,Y)]}, 
\end{equation*}
where $d(X,Y)$ is a distance between random variables $X$ and $Y$, $X \sim \mu$, $Y \sim v$, and the infimum is taken over all joint distributions $\Pi$ with marginals $\mu$ and $v$.

When the action space is finitely discrete, we have the following relationship between the total variation metric and the Wasserstein metric:  $d_{\min} d_{TV} (\pi'(\cdot|s),\pi(\cdot,s)) \le d_{W}(\pi'(\cdot|s),\pi(\cdot,s))$ where $d_{\min} = \min_{a \ne a' \in \mathcal{A}}d(a, a')$ \cite{gibbs2002choosing}. Based on this relationship and Theorem \ref{trpo_theorem}, we obtain the following corollary. 

\begin{corollary}
When the action space is finite discrete, let $c_1 =  \frac{c}{d^2{\min}}$ and $d_{W}^{\max}(\pi',\pi) = \max_{s} d_{W} (\pi'(\cdot|s), \pi(\cdot|s))$, then we have:
\begin{equation*}
    \eta(\pi') \ge L_{\pi}(\pi') - c_1 (d_{W}^{\max}(\pi',\pi))^2. 
\end{equation*}
\label{cor_wass_can_be_used}
\end{corollary}
Let $M_{W}(\pi') = L_{\pi}(\pi') - c_1 (d_{W}^{\max}(\pi',\pi))^2$. From Corollary \ref{cor_wass_can_be_used} we can see that $M_{W}(\pi')$ is a surrogate function that approximates $\eta(\pi')$. Since $M_{W}(\pi) = \eta(\pi)$ and $M_{W}(\pi') \le \eta(\pi')$, $\forall \pi' \neq \pi$, maximizing this surrogate function at each iteration can guarantee a non-decreasing improvement in the true objective function. 

We consider a general state space, and a finite action space with $N$ actions $a_1, a_2 \dots a_N$. The distribution function of policy $\pi$ is represented as $P^{s}_{i} = \pi(a_i|s)$, $\forall s \in \mathcal{S}, i = 1\dots N$. The following two theorems show the dual formulation and the optimal policy distribution of the \texttt{ODRPO} for the Wasserstein metric case (see Appendix A for detailed proofs). 

\begin{restatable}{thm}{thmdualwass}
If Assumption \ref{bounded_A} holds, then the Wasserstein trust-region constrained optimization problem in \eqref{odrpo_problem} is equivalent to the following problem: 
\begin{equation}
\min_{\beta \ge 0} \{\beta\delta + \mathbb{E}_{s \sim \rho^{\pi}} \sum_{i=1}^N P^s_{i}(A^{\pi} (s,a_{j^*_{s,i,\beta}}) -  \beta d(a_{j^*_{s,i,\beta}}, a_i)) \},
\label{wass_trco_dual}
\end{equation}
where $j^*_{s,i, \beta} = \text{argmax}_{j = 1\dots N} ( A^{\pi} (s,a_j) -  \beta d(a_j, a_i) )$.
\label{prop_dual_problem_wass}
\end{restatable}

\begin{restatable}{thm}{thmoptpolicywass}
If Assumption \ref{bounded_A} holds and $\beta^*$ is the global optimal solution to \eqref{wass_trco_dual}, then the optimal policy solution to the Wasserstein trust-region constrained optimization problem in \eqref{odrpo_problem} is: 
\begin{equation}
\pi'^*(a_j|s) = \sum_{i=1}^N P^s_{i} Q^{s*}_{ij},
\label{discrete_wass_optimal}
\end{equation}
where $Q^{s*}_{ij} = 1$ if $j = j^*_{s,i,\beta^*}$ and $Q^{s*}_{ij} = 0$ otherwise. 
\label{thm_optimal_policy_wass}
\end{restatable}

For the Wasserstein metric case, we are not able to obtain the closed-form gradient of the objective in \eqref{wass_trco_dual}, in search of the global optimal $\beta^*$. However, if an initial point $\beta_0$ is given, the closed-form expression of the local optimal solution can still be obtained. We define the distance $d(x,y) = 0$ if $x = y$ and $1$ otherwise, and $k_s = \text{argmax}_{i = 1 \dots N} A^{\pi} (s, a_i)$. Then, we have the following proposition to show the optimal value of $\beta$ regarding different initial points: 

\begin{restatable}{prop}{propoptimalbeta}
\label{prop_optimal_beta_wass}

(1). If the initial point $\beta_0$ is in $[\max_{s, i} \{A^{\pi}(s,a_{k_s}) - A^{\pi} (s,a_i) \}, +\infty)$, the optimal $\beta$ solution is $\max_{s, i} \{A^{\pi}(s,a_{k_s}) - A^{\pi} (s,a_i) \}$. 
\\
(2). If the initial point $\beta_0$ is in $[0, \min_{s, i \ne k_s} \{A^{\pi}(s,a_{k_s}) - A^{\pi} (s,a_i) \}]$: if $\delta - \int_{s\in \mathcal{S}} \rho^{\pi} (s)  (1 -  P^{s}_{k_s}) ds < 0$, the optimal $\beta$ is $\min_{s, i \ne k_s} \{A^{\pi}(s,a_{k_s}) - A^{\pi} (s,a_i) \}$; otherwise, the optimal $\beta$ solution is $0$. 

(3). If the initial point $\beta_0$ is in $(\min_{s, i \ne k_s} \{A^{\pi}(s,a_{k_s}) - A^{\pi} (s,a_i) \}, \max_{s, i} \{A^{\pi}(s,a_{k_s}) - A^{\pi} (s,a_i) \})$, we construct sets $I_s^1$ and $I_s^2$ as: 
$\textbf{\textup{for}} \hspace{1mm} s \in \mathcal{S}, i \in \{1,2\dots N\} \hspace{1mm}: \hspace{1mm} \textbf{\textup{if}} \hspace{1mm} \beta_0 \ge A^{\pi}(s,a_{k_s}) - A^{\pi} (s,a_i) \hspace{1mm} \textbf{\textup{then}} \hspace{1mm} \textup{\text{Add}}\hspace{1mm} i \hspace{1mm} \textup{\text{to}} \hspace{1mm} I_s^1 \hspace{1mm} \textbf{\textup{else}} \hspace{1mm} \textup{\text{Add}} \hspace{1mm} i \hspace{1mm} \textup{\text{to}} \hspace{1mm} I_s^2$. Then, if $\delta - \mathbb{E}_{s \sim \rho^{\pi}}\sum_{i \in I_s^2} P^s_{i} < 0$, the optimal $\beta$ is $\min_{s \in \mathcal{S}, i \in I_s^2} \{ A^{\pi}(s,a_{k_s}) - A^{\pi} (s,a_i) \}$; otherwise, the optimal $\beta$ is $\max_{s \in \mathcal{S}, i \in I_s^1} \{A^{\pi}(s,a_{k_s}) - A^{\pi} (s,a_i)\}$. \end{restatable}

Based on Proposition \ref{prop_optimal_beta_wass}, we can use a two-phase global optimization method \cite{schoen2002_twophase} to find the global optimal solution $\beta^*$. 

\section{ODRPO algorithm and discussion}
\label{drtrpo_alg_sec}

In this section, we first propose our \texttt{ODRPO} algorithm based on the previous results. Then we highlight the major advantages of the proposed \texttt{ODRPO} algorithm as compared to the original TRPO algorithm. 

\subsection{ODRPO algorithm}
Let $\pi_{k+1} = \mathbb{F} (\pi_k)$ denotes the policy update in \eqref{continuous_KL_optimal} if the KL divergence is used to construct the trust region, and let $\pi_{k+1} = \mathbb{G} (\pi_k)$ denotes the policy update in \eqref{discrete_wass_optimal} if the Wasserstein metric is used. Based on Theorems \ref{thm_optimal_policy_kl} and \ref{thm_optimal_policy_wass}, we develop the following \texttt{ODRPO} algorithm: 

\begin{algorithm}[H]
\SetAlgoLined
Input: number of iterations $K$, learning rate $\alpha$ \\
Initialize policy $\pi_0$ and value network $V_{\psi_0}$ with random parameter $\psi_0$\\
 \For{$k = 0,1,2 \dots K$}{
  Collect a set of trajectories $\mathcal{D}_k$ on policy $\pi_k$
  \\
  For each timestep $t$ in each trajectory, compute total returns $G_t$ \\
  \hspace{3mm} and estimate advantages $\hat{A}_t^{\pi_k}$
  \\
  Update value: \\
  \hspace{2cm} $\begin{aligned} \psi_{k+1} \xleftarrow[]{} \psi_{k} - \alpha \nabla_{\psi_{k}} \sum (G_t - V_{\psi_k}(s_t))^2 \end{aligned}$
  \\
  Update policy: \\
  \hspace{2cm} $\pi_{k+1} \xleftarrow[]{} \mathbb{F} (\pi_k)$ \\
  \hspace{1.5cm} \text{or} \hspace{0.5mm} $\pi_{k+1} \xleftarrow[]{} \mathbb{G} (\pi_k)$
 }
 \caption{\texttt{ODRPO} Algorithm}
 \label{odrpo_algorithm}
\end{algorithm}

The trajectories sampled in the algorithm can either be complete or partial. If the trajectory is complete, we obtain the total return by using the accumulated discounted rewards $G_t  = \sum_{k=0}^{T-t-1} \gamma^k r_{t+k}$. If the trajectory is partial, we estimate the total return by using the multi-step temporal difference (TD) methods \cite{asis2017_multistep}: $\hat{G}_{t:t+n} = \sum_{k=0}^{n-1} \gamma^k r_{t+k} + \gamma^n V(s_{t+n})$. In the advantage estimation step, we can use the Monte Carlo advantage estimation, i.e., $\hat{A}_t^{\pi_k}=G_t - V_{\psi_k}(s_t)$ or alternative advantage estimation methods such as Generalized Advantage Estimation (GAE) \cite{schulmanl2016_gae_continuous}, which provides an explicit control over the bias-variance trade-off. 

\subsection{Same stability, better sample efficiency and better performance}

The original TRPO approximately solves the trust-region constraint optimization problem in \eqref{trpo_problem} by using the conjugate gradient algorithm, after making a linear approximation to the objective and a quadratic approximation to the constraint. These approximations will not prevent the method from performing large steps of policy update that violate the KL constraint, and therefore cause the degradation of the system performance. To overcome it, TRPO adds a line search which can reject the update that either violates the KL constraint or decreases the objective of the surrogate function. However, due to the approximations, the new policy accepted by the line search may not be the optimal policy within the trust region. Another factor that may lead TRPO to perform sub-optimal policy updates and converge to a sub-optimal final policy is that in TRPO, the policy is assumed to belong to a particular parametric distribution class (i.e., $\pi(a|s) = P [a|s; \theta]$) \cite{tessler2019_dpo}. Since the optimal policy does not necessarily belong to this particular distribution class, it may be excluded when TRPO performs policy updates.

The proposed \texttt{ODRPO} algorithm overcomes the above limitations by utilizing DRO. At each iteration of the \texttt{ODRPO} policy update, we use DRO to exactly solve the trust-region constraint optimization problem in \eqref{odrpo_problem}, without making any approximation. Also, in DRO, we consider all admissible policies (i.e., distributional policies within the trust region). Therefore the new policy is always the one that maximizes the surrogate objective within the trust region centered on the old policy. Thus, the trust-region constraint is satisfied in each update and the globally optimal policy update can be obtained in each iteration. 

To conclude, \texttt{ODRPO} has the following advantages as compared with TRPO: 
\begin{itemize}
    \item \texttt{ODRPO} maintains the stability of TRPO since the trust region constraint is always satisfied.
    \item \texttt{ODRPO} has a better sample efficiency since it performs updates more often (accepts all updates), and in each update, the policy obtained by \texttt{ODRPO} is globally optimal within the trust region.
    \item \texttt{ODRPO} converges to a better final policy since \texttt{ODRPO} learns distributional policies while TRPO limits policies to a particular parametric distribution class. 
\end{itemize}

\section{Experiments}
\label{experiment_section}
In this section, we conduct an experimental evaluation of \texttt{ODRPO} on tabular domains and robotic locomotion tasks. We use \texttt{ODRPO-KL} to denote the \texttt{ODRPO} algorithm with KL divergence and \texttt{ODRPO-Wass} to denote the \texttt{ODRPO} algorithm with Wasserstein metric. We compare the performance of our proposed method with several benchmarks, including TRPO \cite{schulman2015_trpo}, PPO \cite{schulman2017_ppo} and A2C \cite{mnih2016_a2c} \footnote{We use the implementations of TRPO, PPO and A2C from OpenAI Baselines \cite{openaibaseline} for MuJuCo tasks and Stable Baselines \cite{stable-baselines} for other tasks. }
. We consider comparing with A2C because it is similar to our \texttt{ODRPO} algorithm in the sense that both of them are simple on-policy actor-critic methods that utilize the advantage information to perform policy updates. 
\subsection{Tabular domains}
We evaluate \texttt{ODRPO} on a set of tasks including Taxi \cite{dietterich1998_taxi}, Chain \cite{richard1998_chain} and Cliff Walking \cite{sutton2018_cliff_walking}, which are intentionally designed to test the exploration ability of the algorithms. The tabular domain has a special environment structure with a discrete state space and a discrete action space. Thus, we use an array of size $|\mathcal{S}| \times |\mathcal{A}|$ to represent the policy $\pi(a|s)$. For the value function, we use a neural net to smoothly update the values. The performance of \texttt{ODRPO} is compared to the performance of TRPO, PPO and A2C with the same neural net structure. Each algorithm is evaluated $5$ times with a random initialization.

As shown in Figure \ref{fig:tabular}, the performance of \texttt{ODRPO} and TRPO is manifestly better than the first-order optimization methods A2C and PPO. Between the trust region based methods, \texttt{ODRPO} outperforms TRPO in most tasks, except in Chain, where the performance of these two methods are not significantly different. In Cliff Walking, \texttt{ODRPO} converges to the optimum much faster than TRPO. In Taxi, \texttt{ODRPO} converges to a better optimum compared to TRPO. We further analyze the performance of the trained agent for each algorithm on the Taxi environment. As shown in Table \ref{taxi_table}, \texttt{ODRPO} has a higher successful drop-off rate and a lower task completion time while the original TRPO reaches the time limit with a drop-off rate $0$. Therefore, the results in Taxi show that \texttt{ODRPO} finds a better policy than the original TRPO. Among the variants of \texttt{ODRPO}, the relative strength of \texttt{ODRPO-KL} and \texttt{ODRPO-Wass} on convergence rate and solution quality is highly task dependent. However, a general pattern across tasks shows that \texttt{ODRPO-KL} converges faster and attains a slightly better optimum than \texttt{ODRPO-Wass} in the tabular domains.

\begin{table}[H]
\renewcommand{\arraystretch}{1.2}
\centering
\caption{Trained agents performance on Taxi (averaged over $1000$ episodes)}
\begin{adjustbox}{width=0.9\columnwidth,center}
\begin{tabular}{lllll}
 \hline
 & Successful Dropoff (+20) & Illegal Pickup/ Dropoff (-10) &  Timesteps (-1) & Average Return \\
 \hline
 \texttt{ODRPO-KL} & 0.995 & 0.732 & 18.556 & -5.976 \\
 \hline
 \texttt{ODRPO-Wass} & 0.753  & 0.232 & 70.891 & -58.151  \\
 \hline
 TRPO & 0 & 0 & 200 & -200   \\
 \hline
\end{tabular}
\label{taxi_table}
\end{adjustbox}
\end{table}

\begin{figure}[H]
\centering
\begin{minipage}{.3\textwidth}
  \centering
  \includegraphics[width=0.9\linewidth]{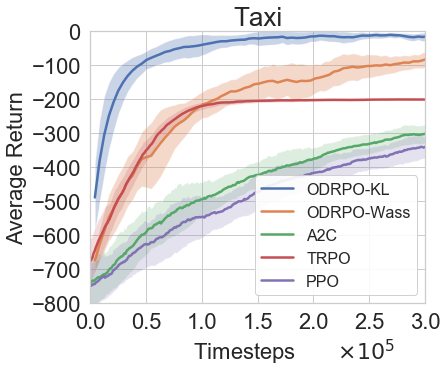}
\end{minipage}
\begin{minipage}{.3\textwidth}
  \centering
  \includegraphics[width=0.9\linewidth]{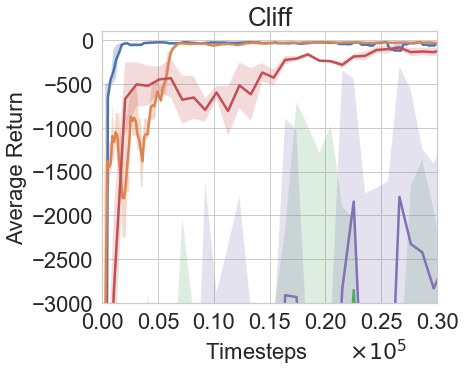}
\end{minipage}
\begin{minipage}{.3\textwidth}
  \centering
  \includegraphics[width=0.9\linewidth]{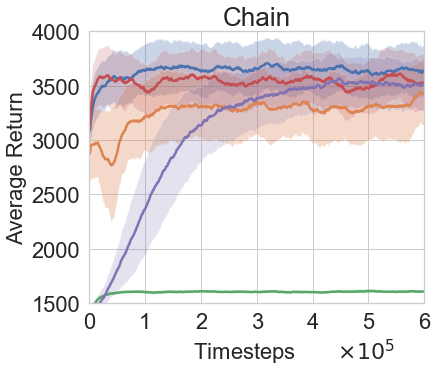}
\end{minipage}
\caption{Episode rewards during the training process for the tabular domain tasks, averaged across $5$ runs with a random initialization. The shaded area depicts the mean $\pm$ the standard deviation.}
\label{fig:tabular}
\end{figure}

\subsection{Robotic locomotion tasks}
We then integrate deep neural network architecture into the \texttt{ODRPO} framework and evaluate its performance on several locomotion tasks, including CartPole \cite{barto1983_carpole}, Acrobot \cite{geramifard2015_acrobot}, Lunar Lander \cite{brockman2016_openai_gym} and MuJuCo control suite \cite{todorov2012_mujuco}. For the experiments with a discrete action space, we use two separate neural nets to represent the policy and the value. The policy neural net receives state $s$ as an input and outputs the categorical distribution of $\pi(a|s)$. For the experiments with a continuous action space, we use the generative actor critic network in \cite{tessler2019_dpo} with the target distribution defined by $\pi_{k+1} = \mathbb{F} (\pi_k)$. The performance of \texttt{ODRPO} is compared to the performance of TRPO, PPO and A2C with the same neural net structure. Each algorithm is run $5$ times with a random initialization.

\begin{wraptable}{r}{0.45\textwidth}
\caption{Timesteps to hit threshold ($\times 10^4$)}
\renewcommand{\arraystretch}{1.1}
\centering
\begin{adjustbox}{width=\linewidth,center}
\begin{tabular}{lllll}
 \hline
 & Threshold & \texttt{ODRPO}  & TRPO & PPO \\
 \hline
CartPole & 200 & 0.8 & 0.7 & 2.3  \\
Acrobot & -150 & 5.1 & 15.0 & 2.0  \\
LunarLander & 50 & 10 & 18 &  8 \\
Walker & 3000 & 45 & 112 & 52  \\
Hopper & 3000 & 64 & 115 & 127  \\
Reacher & -10 & 2 & 13 & 20  \\
 \hline
\end{tabular}
\label{locomotion_efficiency_table}
\end{adjustbox}
\end{wraptable}

\textbf{Sample efficiency: }
Table \ref{locomotion_efficiency_table} lists the timesteps required by \texttt{ODRPO}, TRPO and PPO to hit a predefined threshold. As seen in Table \ref{locomotion_efficiency_table}, in most cases, \texttt{ODRPO} reduces the number of timesteps by at least half as compared with TRPO. Also \texttt{ODRPO} requires less timesteps than PPO in most tasks. 

\textbf{Final performance/quality: }
Figure \ref{fig:locomotion_control} shows the episode rewards during the training process and Table \ref{locomotion_performance_table} shows the performance improvement of the best among two \texttt{ODRPO} methods compared to the best among the baseline algorithms: TRPO, PPO, and A2C. As seen in Figure \ref{fig:locomotion_control}, \texttt{ODRPO} outperforms TRPO, PPO and A2C in most tasks in terms of the final performance. As shown in Table \ref{locomotion_performance_table}, \texttt{ODRPO} achieves $+75\%$ better performance in average compared to the best of TRPO, PPO and A2C. 

\begin{table}[H]
\footnotesize 
\renewcommand{\arraystretch}{1.25}
\caption{\texttt{ODRPO} best performance compared to the baseline best performance}
\begin{adjustbox}{width=\columnwidth,center}
\centering
\begin{tabular}{lllllllll}
 \hline
  CartPole  & Acrobot  & LularLander & Walker & Hopper & Reacher & Ant & HalfCheetah & InvertedPendulum \\
 \hline
  $+94$  & $-12$ & $+77$ &  $+507$ & $+68$  & $+1$ & $+3647$ & $+6330$ & $+11$\\
  $(+30\%)$ & $(-14\%)$ &  $(+82\%)$ & $(+16\%)$ & $(+3\%)$ & $(+13\%)$ & $(+208\%)$ & $(+343\%)$ & $(+1\%)$\\
 \hline
\end{tabular}
\label{locomotion_performance_table}
\end{adjustbox}
\end{table}

\textbf{Training time: } 
To train $10^5$ timesteps in the discrete locomotion tasks, the training wall-clock time is around 63s for \texttt{ODRPO}, 59s for TRPO and 70s for PPO. Therefore, \texttt{ODRPO} has a similar computational efficiency as TRPO and PPO. 

Among variants of \texttt{ODRPO}, \texttt{ODRPO-KL} converges faster than \texttt{ODRPO-Wass}. This is coherent with the convergence property discussed in \cite{hu2012_dro_kl} and \cite{kuhn2015_dro_wasserstein} that the radii of Wasserstein uncertainty set decays polynomially with a larger number of samples while the radii of KL uncertainty set decays exponentially. As a result, \texttt{ODRPO-KL} has a smaller uncertainty set than \texttt{ODRPO-Wass}, and thus converges faster.

\begin{figure}[H]
\centering
\begin{minipage}{.245\textwidth}
  \centering
  \includegraphics[width=\linewidth]{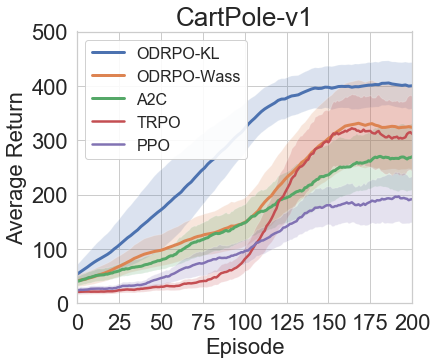}
\end{minipage}
\begin{minipage}{.245\textwidth}
  \centering
  \includegraphics[width=\linewidth]{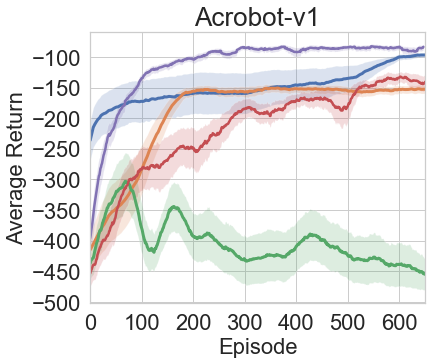}
\end{minipage}
\begin{minipage}{.245\textwidth}
  \centering
  \includegraphics[width=\linewidth]{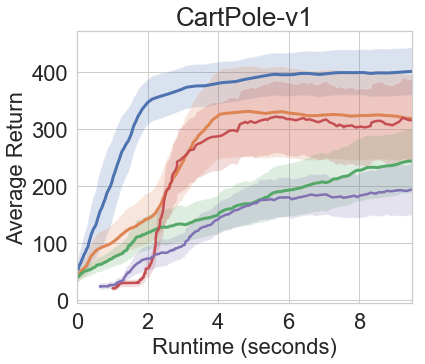}
\end{minipage}
\begin{minipage}{.245\textwidth}
  \centering
  \includegraphics[width=\linewidth]{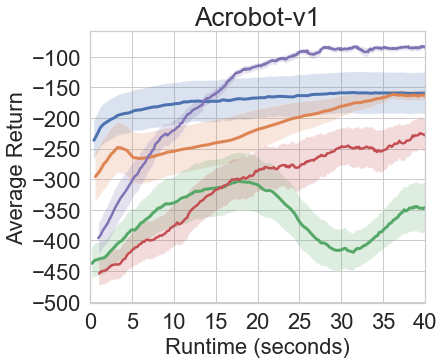}
\end{minipage}
\begin{minipage}{.245\textwidth}
  \centering
  \includegraphics[width=\linewidth]{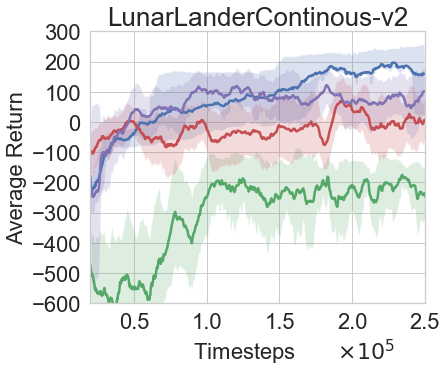}
\end{minipage}
\begin{minipage}{.245\textwidth}
  \centering
  \includegraphics[width= \linewidth]{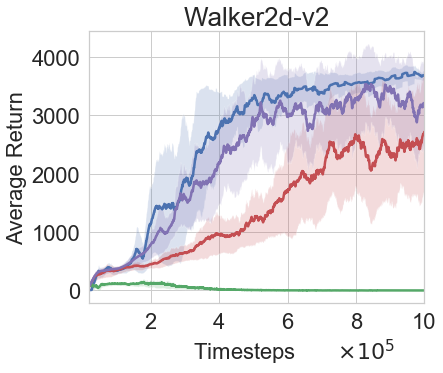}
\end{minipage}
\begin{minipage}{.245\textwidth}
  \centering
  \includegraphics[width= \linewidth]{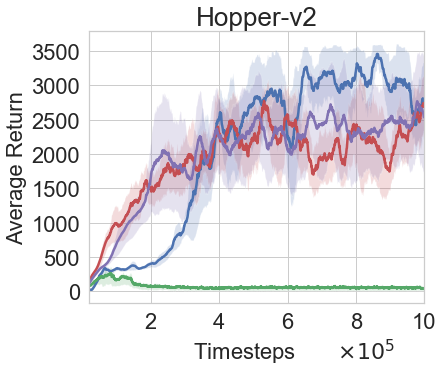}
\end{minipage}
\begin{minipage}{.245\textwidth}
  \centering
  \includegraphics[width= \linewidth]{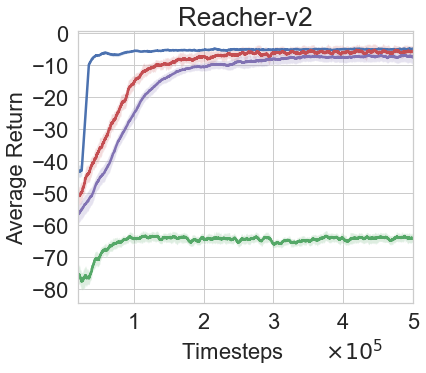}
\end{minipage}
\caption{Episode rewards during the training process for the locomotion tasks, averaged across $5$ runs with a random initialization. The shaded area depicts the mean $\pm$ the standard deviation.}
\label{fig:locomotion_control}
\end{figure}

\section{Discussion and future work}
In this paper, we present a novel RL algorithm, Optimistic Distributionally Robust Policy Optimization (\texttt{ODRPO}), which addresses the sub-optimality issue of trust region based methods such like TRPO and PPO. Our \texttt{ODRPO} method improves TRPO and PPO with better sample efficiency and better final performance, while maintaining the stable learning property. 

To our knowledge, our \texttt{ODRPO} method is the first trust region based policy optimization method that learns distributional policies. We believe that this work can be a principal alternative to the trust region based PG methods. \texttt{ODRPO} is a theoretically grounded method, but at the same time, it is a simple on-policy actor-critic algorithm. Therefore, we will focus on integrating \texttt{ODRPO} with other well-known off-policy frameworks, such as Soft Actor Critic \cite{haarnoja2018_sac} and Twin Delayed DDPG \cite{fujimoto2018_td3} as our future work. 

\newpage
\section{Broader impact}
By employing the probabilistic model-free conception into the reinforcement learning domain, we open the opportunities for researchers and practitioners to achieve better learning outcomes. The reason is three-fold.
First of all, by replacing untrustworthy probabilistic assumptions with a confidence set that covers all permissible distributions, a potentially better result is theoretically guaranteed. Second, by avoiding having to resort to heuristics, each step can be solved to its optimality using a rigorous optimization approach. Third, it greatly reduces the effort of tackling nonconvexity constraints introduced by parametric policy assumptions, which is oftentimes the major computational obstacles in solving the subroutine of the model-free reinforcement learning framework. 

% It is permissible to reduce the font size to small (9 point) when listing the references.
\small
\bibliographystyle{unsrt}
\bibliography{ms}

\newpage
% Appendix follows normal font size
\normalsize
% Disable line numbers in Appendix
% \nolinenumbers
\appendix

\section{Theorem proof in section \ref{ODRPO_section_has_proofs}}
\label{appendix_proof}

\thmdualkl*
\begin{proof}[Proof of Theorem \ref{prop_dual_problem_kl}]

Let $\begin{aligned} L_s(a) = \frac{\pi'(a|s)}{\pi(a|s)} \end{aligned}$. Denote $\mathbb{L}_s = \{L_s' \hspace{1mm} | \hspace{1mm} \mathbb{E}_{a \sim \pi(\cdot |s)} [L_s'(a)] = 1, \hspace{1mm} L_s' \ge 0\}$. It's easy to prove that $L_s \in \mathbb{L}_s$. By using the importance sampling and the definition of KL divergence, we have:

\begin{equation*}
\begin{split}
     & \mathbb{E}_{ a \sim \pi'(\cdot|s)} [ A^{\pi} (s,a)] = \mathbb{E}_{a \sim \pi(\cdot|s)} [ A^{\pi} (s,a) L_s(a)], \\
     d_{KL} (\pi'(\cdot|s), & \pi(\cdot|s)) = \mathbb{E}_{a \sim \pi'(\cdot|s)} [\log{L_s(a)}] = \mathbb{E}_{a \sim \pi (\cdot|s)} [L_s(a) \log{L_s(a)}]. 
\end{split}
\end{equation*}

Thus, we can reformulate \eqref{odrpo_problem} with KL divergence based trust region as: 

\begin{equation}
\begin{split}
& \max_{L_s \in \mathbb{L}_s} \hspace{3mm} \mathbb{E}_{s\sim \rho^{\pi}} \mathbb{E}_{a \sim \pi(\cdot|s)} [ A^{\pi} (s,a) L_s(a)]  \\
& s.t. \hspace{3mm} \mathbb{E}_{s\sim \rho^{\pi}} \mathbb{E}_{a \sim \pi(\cdot|s)} [L_s(a) \log L_s(a)] \le \delta. 
\end{split}
\label{continuous_KL_eq2}
\end{equation}

First, it is easy to prove that for $\forall s, a$, $\mathbb{E}_{a \sim \pi(\cdot |s)}[L_s(a)]$ and $\mathbb{E}_{s\sim \rho^{\pi}} \mathbb{E}_{a \sim \pi(\cdot|s)} [ A^{\pi} (s,a) L_s(a)]$ are linear functions of $L_s(a)$, and $\mathbb{E}_{s\sim \rho^{\pi}} \mathbb{E}_{a \sim \pi(\cdot|s)} [L_s(a) \log L_s(a)]$ is a convex function of $L_s(a)$. In addition, Slater's condition holds for \eqref{continuous_KL_eq2} since there is an interior point $L_s(a) = 1$  $\forall s,a$. Meanwhile, since $A^{\pi} (s,a)$ is bounded following from Assumption \ref{bounded_A}, the objective is bounded above. Therefore, strong duality holds for \eqref{continuous_KL_eq2}. To reformulate \eqref{continuous_KL_eq2}, we consider its Lagrangian duality function: \\
\begin{equation*}
    \begin{split}
    l_0(\beta, L_s) &=   \mathbb{E}_{s\sim \rho^{\pi}}  \mathbb{E}_{a \sim \pi(\cdot|s)} [ A^{\pi} (s,a) L_s(a)]   - \beta\{\mathbb{E}_{s\sim \rho^{\pi}} \mathbb{E}_{a \sim \pi(\cdot|s)} [L_s(a) \log L_s(a)] - \delta \} \\
    & =   \mathbb{E}_{s\sim \rho^{\pi}} \mathbb{E}_{a \sim \pi(\cdot|s)}[ A^{\pi} (s,a) L_s(a) - \beta L_s(a) \log L_s(a)] +  \beta\delta,
    \end{split}
\end{equation*}

where $\beta$ is the dual variable. Then, \eqref{continuous_KL_eq2} is equivalent to its dual problem as follows: 
\begin{equation}
    \min_{\beta \ge 0} \hspace{1mm} \max_{L_s \in \mathbb{L}_s} \hspace{1mm} l_0(\beta, L_s ).
    \label{continuous_KL_dual_eq1}
\end{equation}

The inner maximization problem of \eqref{continuous_KL_dual_eq1} is equivalent to: 
\begin{equation}
    \begin{split}
      \max_{L_s \ge 0} \hspace{3mm} \mathbb{E}_{s\sim \rho^{\pi}} &\mathbb{E}_{a \sim \pi(\cdot|s)} [ A^{\pi} (s,a) L_s(a) - \beta L_s(a) \log L_s(a) ] + \beta\delta 
     \\ & s.t. \hspace{3mm} \mathbb{E}_{a \sim \pi(\cdot|s)} [L_s(a)] = 1, \hspace{5mm} \forall s \in \mathcal{S}.
    \end{split}
    \label{continuous_KL_dual_eq2}
\end{equation}

By Theorem 1 of \cite{hu2012_dro_kl}, we can obtain the optimal solution $L_s^*$ and the optimal objective value of the inner maximization problem \eqref{continuous_KL_dual_eq2} respectively as follows:
\begin{equation*}
    L_s^* (a) = \frac{e^{ A^\pi (s,a)/\beta}}{\mathbb{E}_{a \sim \pi(\cdot|s)} [e^{A^\pi (s,a)/\beta}]}, 
\label{inner_L_solution}
\end{equation*}

\begin{equation*}
l_0(\beta, L_s^* ) = \beta\delta + \mathbb{E}_{s\sim \rho^{\pi}} \beta \log \mathbb{E}_{a \sim \pi(\cdot|s)} [e^{A^\pi (s,a)/\beta}]. 
\end{equation*}

Therefore, \eqref{continuous_KL_dual_eq1} can be further reformulated as: 
\begin{equation*}
     \min_{\beta \ge 0} l_0(\beta, L_s^* ) = \min_{\beta \ge 0} \{\beta\delta  + \mathbb{E}_{s\sim \rho^{\pi}} \beta \log \mathbb{E}_{a \sim \pi(\cdot|s)} [e^{A^\pi (s,a)/\beta}]\}.
\end{equation*}
\end{proof}

\thmoptpolicykl*
\begin{proof}[Proof of Theorem \ref{thm_optimal_policy_kl}]
Based on the proof of Theorem \ref{prop_dual_problem_kl}, the optimal solution $L_s^* (a)$ to \eqref{continuous_KL_dual_eq1} is: $\begin{aligned}L_s^* (a) = \frac{e^{ A^\pi (s,a)/\beta^*}}{\mathbb{E}_{a \sim \pi(\cdot|s)} [e^{A^\pi (s,a)/\beta^*}]} \end{aligned}$. Since $\pi'^*(a|s) = L_s^* (a)\pi(a|s)$, we have \eqref{continuous_KL_optimal} holds.
\end{proof}

\thmdualwass*
\begin{proof}[Proof of Theorem \ref{prop_dual_problem_wass}]

Assume $q_s(\cdot, \cdot)$ is a joint distribution with marginals $\pi(\cdot|s)$ and $\pi'(\cdot|s)$. Let $Q_i^s(a)$ represent the conditional distribution of $\pi'(a|s)$ under $\pi(a_i|s)$, and $Q_{ij}^s= Q_i^s(a_j)$. Then $q_s (a_i,a_j) = P_i^s Q_{ij}^s$, $\pi'(a_j|s) = \sum_{i=1}^N q_s (a_i,a_j) = \sum_{i=1}^N P^s_{i} Q^s_{ij}$. In addition:  

\begin{eqnarray*}
&d_W (\pi'(\cdot|s), \pi(\cdot|s)) & = \sum_{i=1}^N \sum_{j=1}^N d(a_j, a_i) q_s (a_i,a_j) =  \sum_{i=1}^N \sum_{j=1}^N d(a_j, a_i) P^s_{i} Q^s_{ij},\label{Was1}\\
&\mathbb{E}_{a \sim \pi'(\cdot|s)} [ A^{\pi} (s,a)] & = \sum_{j=1}^N   A^{\pi} (s,a_j) \pi'(a_j|s) = \sum_{i=1}^N \sum_{j=1}^N  A^{\pi} (s,a_j) P^s_{i} Q^s_{ij}.\label{Was2}
\end{eqnarray*}

 Thus, \eqref{odrpo_problem} with Wasserstein metric based trust region can be reformulated as: 
\begin{equation}
\begin{split}
 \max_{Q^s_{ij} \ge 0} \hspace{3mm} \mathbb{E}_{s\sim \rho^{\pi}}  &\sum_{i=1}^N \sum_{j=1}^N  A^{\pi} (s,a_j) P^s_{i} Q^s_{ij} \\
 s.t.  \hspace{3mm} \mathbb{E}_{s\sim \rho^{\pi}}  &\sum_{i=1}^N \sum_{j=1}^N d(a_j, a_i) P^s_{i} Q^s_{ij}  \le \delta, \\
 &  \hspace{3mm} \sum_{j=1}^N  Q^s_{ij} = 1, \hspace{1cm} \forall s \in \mathcal{S}, i = 1 \dots N.
\end{split}
\label{discrete_wass_eq2}
\end{equation}

Both the objective function and the constraint are linear in $Q^s_{ij}$. Thus, \eqref{discrete_wass_eq2} is a convex optimization problem. Slater's condition holds for \eqref{discrete_wass_eq2} because the feasible region has an interior point, which is $Q_{ii}^s = 1$  $\forall i$, $Q_{ij}^s = 0$  $\forall i \ne j$. Meanwhile, since $A^{\pi} (s,a)$ is bounded based on Assumption \ref{bounded_A}, the objective is bounded above. Thus, strong duality holds. Therefore, \eqref{discrete_wass_eq2} is equivalent to its dual problem as follows:
 \begin{equation}
\begin{split}
    & \min_{\beta \ge 0, \lambda^s_{i}} \hspace{3mm} \beta\delta + \int_{s \in \mathcal{S}} \sum_{i=1}^N \lambda^s_{i} ds\\
    & s.t. \hspace{3mm} A^{\pi} (s,a_j) P^s_{i} - \beta d(a_j, a_i) P^s_{i} - \frac{\lambda^s_{i}}{\rho^{\pi}(s)}  \le 0, \hspace{1cm} \forall s \in \mathcal{S}, i, j = 1 \dots N.
\label{discrete_wass_dual_eq1}
\end{split}
\end{equation}

We observe that with a fixed $\beta$, the optimal $\lambda^s_{i}$ is: \begin{equation*}\lambda^{s*}_{i}(\beta) = \max_{j = 1 \dots N} \rho^{\pi}(s)P^s_{i}(A^{\pi} (s,a_j) -  \beta d(a_j, a_i)).\end{equation*}

Thus, the dual problem can be further reformulated as:
\begin{equation*}
    \min_{\beta \ge 0} \{\beta\delta + \int_{s \in \mathcal{S}} \sum_{i=1}^N \lambda^{s*}_{i}(\beta) ds \} = \min_{\beta \ge 0} \{\beta\delta + \mathbb{E}_{s \sim \rho^{\pi}} \sum_{i=1}^N P^s_{i}(A^{\pi} (s,a_{j^*_{s,i,\beta}}) -  \beta d(a_{j^*_{s,i,\beta}}, a_i)) \},
\end{equation*}
where $j^*_{s,i, \beta} = \text{argmax}_{j = 1\dots N} ( A^{\pi} (s,a_j) -  \beta d(a_j, a_i) )$. 
\end{proof}

\thmoptpolicywass*
\begin{proof}[Proof of Theorem \ref{thm_optimal_policy_wass}]
The optimal solution $\lambda^{s*}_{i}$ to \eqref{discrete_wass_dual_eq1} is: 
\begin{equation*} \lambda^{s*}_{i} = \lambda^{s*}_{i}(\beta^*) = \max_{j = 1 \dots N} \rho^{\pi}(s)P^s_{i}(A^{\pi} (s,a_j) -  \beta^* d(a_j, a_i)). \end{equation*}
Due to the complimentary slackness we have:
\begin{equation*}
    (A^{\pi} (s,a_j) P^s_{i} - \beta^* d(a_j, a_i) P^s_{i} - \frac{\lambda^{s*}_{i}}{\rho^{\pi}(s)} ) Q^{s*}_{ij} = 0, \hspace{1cm} \forall s,i,j.
\label{wass_optimal_condition2}
\end{equation*}
In this case, $Q^{s*}_{ij}$ is non-zero only when $A^{\pi} (s,a_j) P^s_{i} - \beta^* d(a_j, a_i) P^s_{i} - \frac{\lambda^{s*}_{i}}{\rho^{\pi}(s)}  = 0$, which happens when $j  = j^*_{s,i,\beta^*} = \text{argmax}_{j = 1 \dots N}  \hspace{3mm} ( A^{\pi} (s,a_j) -  \beta^* d(a_j, a_i))$. Since such $j$ is unique for each pair of $s,i$ and $\sum_{j=1}^N  Q^{s*}_{ij} = 1$, the optimal solution $Q^{s*}_{ij}$ is:

\[ Q^{s*}_{ij} = \begin{cases*}
                    1 & \hspace{3mm} if  $j = j^*_{s,i,\beta^*}$  \\
                    0 & \hspace{3mm} otherwise.
                 \end{cases*} \]%
                 
Since $\pi'^{*}(a_j|s) = \sum_{i=1}^N P^s_{i} Q^{s*}_{ij}$, we have \eqref{discrete_wass_optimal} holds. 
\end{proof}

\propoptimalbeta*

\begin{proof}[Proof of Proposition \ref{prop_optimal_beta_wass}]

(1). When $\beta \in [\max_{s, i} \{A^{\pi}(s,a_{k_s}) - A^{\pi} (s,a_i) \}, +\infty)$, we have $A^{\pi} (s,a_i) \ge A^{\pi}(s,a_{k_s}) - \beta$ for all $s \in \mathcal{S}$, $i = 1 \dots N$. Since $A^{\pi}(s,a_{k_s}) - \beta \ge A^{\pi}(s,a_{j}) - \beta$ for all $j = 1 \dots N$, we have $A^{\pi} (s,a_i) \ge A^{\pi}(s,a_{j}) - \beta$  for all  $s \in \mathcal{S}$, $i = 1 \dots N$, $j = 1 \dots N$. Thus, $j^*_{s,i, \beta} = i$, for all $s \in \mathcal{S}$,  $i = 1 \dots N$. Thus, \eqref{wass_trco_dual} can be reformulated as: 
\begin{equation*}
    \min_{\beta \ge 0} \{\beta\delta + \mathbb{E}_{s \sim \rho^{\pi}} \sum_{i=1}^N P^s_{i}A^{\pi} (s,a_i) \}.
\end{equation*}
Since $\delta \ge 0$, we have the local optimal $\beta = \max_{s, i} \{A^{\pi}(s,a_{k_s}) - A^{\pi} (s,a_i) \}$.

(2). When $\beta \in [0, \min_{s, i \ne k_s} \{A^{\pi}(s,a_{k_s}) - A^{\pi} (s,a_i) \}]$, we have $A^{\pi} (s,a_i) \le A^{\pi}(s,a_{k_s}) - \beta$ for all $s \in \mathcal{S}$,  $i = 1 \dots N$, $i \ne k_s$. Thus $j^*_{s,i, \beta} = k_s$ for all $s \in \mathcal{S}$,  $i = 1 \dots N$. The inner part of \eqref{wass_trco_dual} then is: 
\begin{equation*}
\begin{gathered}
    \beta\delta + \mathbb{E}_{s \sim \rho^{\pi}} \{ \sum_{i=1, i \ne k_s}^N P^s_{i} ( A^{\pi} (s,a_{k_s}) - \beta ) + P_{k_s}^s A^{\pi} (s,a_{k_s}) \} \\
    = \beta (\delta - \mathbb{E}_{s \sim \rho^{\pi}}  \sum_{i=1, i \ne k_s}^N P^s_{i})+ \mathbb{E}_{s \sim \rho^{\pi}}  \sum_{i=1}^N P^s_{i}  A^{\pi} (s,a_{k_s}) \\
    = \beta (\delta - \int_{s\in \mathcal{S}} \rho^{\pi} (s)  (1 -  P^{s}_{k_s}) ds ) + \mathbb{E}_{s \sim \rho^{\pi}}  \sum_{i=1}^N P^s_{i}  A^{\pi} (s,a_{k_s}).
\end{gathered} 
\end{equation*}
If $\delta - \int_{s\in \mathcal{S}} \rho^{\pi} (s)  (1 -  P^{s}_{k_s}) ds < 0$, we have the local optimal $\beta = \min_{s, i \ne k_s} \{A^{\pi}(s,a_{k_s}) - A^{\pi} (s,a_i) \}$.
If $\delta - \int_{s\in \mathcal{S}} \rho^{\pi} (s)  (1 -  P^{s}_{k_s}) ds \ge 0$, we have the local optimal $\beta = 0$.

(3). For an initial point $\beta_0$ in $(\min_{s, i \ne k_s} \{A^{\pi}(s,a_{k_s}) - A^{\pi} (s,a_i) \}, \max_{s, i} \{A^{\pi}(s,a_{k_s}) - A^{\pi} (s,a_i) \})$, we construct partitions $I_s^1$ and $I_s^2$ of the set $\{1,2 \dots N\}$ in the way described in Proposition \ref{prop_optimal_beta_wass} for all $s \in \mathcal{S}$. Consider $\beta$ in the neightborhood of $\beta_0$, i.e., $\beta \ge A^{\pi}(s,a_{k_s}) - A^{\pi} (s,a_i)$ for $s \in \mathcal{S}, i \in I_s^1$ and $\beta \le A^{\pi}(s,a_{k_s}) - A^{\pi} (s,a_i)$ for $s \in \mathcal{S}, i \in I_s^2$. Then the inner part of \eqref{wass_trco_dual} can be reformulated as: 
\begin{equation*}
\begin{gathered}
    \beta\delta + \mathbb{E}_{s \sim \rho^{\pi}} \{ \sum_{i \in I_s^1} P^s_{i}  A^{\pi} (s,a_{i})  + \sum_{i \in I_s^2} P^s_{i} ( A^{\pi} (s,a_{k_s}) - \beta ) \} \\
    =  \beta (\delta - \mathbb{E}_{s \sim \rho^{\pi}}  \sum_{i \in I_s^2} P^s_{i})+ \mathbb{E}_{s \sim \rho^{\pi}}  \{ \sum_{i \in I_s^1} P^s_{i}  A^{\pi} (s,a_{i}) + \sum_{i \in I_s^2} P^s_{i}  A^{\pi} (s,a_{k_s}) \}. \\
\end{gathered} 
\end{equation*}
If $\delta - \mathbb{E}_{s \sim \rho^{\pi}}\sum_{i \in I_s^2} P^s_{i} < 0$, we have the local optimal $\beta = \min_{s \in \mathcal{S}, i \in I_s^2} \{ A^{\pi}(s,a_{k_s}) - A^{\pi} (s,a_i) \}$. If $\delta - \mathbb{E}_{s \sim \rho^{\pi}}\sum_{i \in I_s^2} P^s_{i} \ge 0$, we have the local optimal $\beta = \max_{s \in \mathcal{S}, i \in I_s^1} \{A^{\pi}(s,a_{k_s}) - A^{\pi} (s,a_i)\}$. 
\end{proof}

\section{Implementation details}
\label{appendix_experimental}
Open Source Code:  \url{https://github.com/kadysongbb/dr-trpo}

\textbf{Visitation frequencies estimation: } The unnormalized discounted visitation frequencies are needed to compute the global optimal $\beta^*$. At the $k$-th iteration, the visitation frequencies $\rho^\pi_k$ are estimated using samples of the trajectory set $\mathcal{D}_k$. Specifically, we first initialize $\rho^\pi_k (s) = 0, \; \forall s \in S$. Then for each timestep $t$ in each trajectory from $\mathcal{D}_k$, we update $\rho^\pi_k$ as $\rho^\pi_k(s_t) \xleftarrow[]{} \rho^\pi_k (s_t) + \gamma^t / |\mathcal{D}_k|$.

\textbf{Global optimal beta estimation: } In order to avoid the computationally cumbersome step of obtaining $\beta^*$ in each iteration, we can approximate that with the $\beta$ value resulted from the first several policy updates. Small perturbations should be added to the approximate values to avoid any stagnation in updating. This approximation is rooted on the observations of the authors that $\beta^*$ does not change significantly throughout all the policy updates in the experiments carried out in the paper. This trick generally reduces the computational time and speeds up the convergence without noticeable performance degradation.

\textbf{Policy representation: } The general approach depicted in Algorithm \ref{odrpo_algorithm} allows various policy representations including arrays and neural networks. Let $\mathcal{S}_k \subseteq \mathcal{S}$ represent a subset of states to perform the policy update at the $k$-th iteration. When an array is used, the policy update step is simply $\pi_{k+1}(\cdot|s) = \mathbb{F}(\pi_{k})(\cdot|s), \; \forall s \in \mathcal{S}_k$ or $\pi_{k+1}(\cdot|s) = \mathbb{G}(\pi_{k})(\cdot|s), \hspace{1.5mm} \forall s \in \mathcal{S}_k$. When a neural network is employed, the policy update step can be achieved by obtaining the gradient descent, i.e., $\begin{aligned} \nabla \sum_{s \in \mathcal{S}_k} (\mathbb{F} (\pi_k)(\cdot|s) - \pi_{k}(\cdot|s))^2 \end{aligned}$ or $\begin{aligned} \nabla \sum_{s \in \mathcal{S}_k} (\mathbb{G} (\pi_k)(\cdot|s) - \pi_{k}(\cdot|s))^2 \end{aligned}$. 

\textbf{Policy updating strategy: } 
\begin{enumerate}[wide]
    \item State space: For environments with a discrete state space (e.g., tabular domains), the policy update is performed for all states at each iteration. For the environments with a continuous state space, a random subset of states is sampled at each iteration to perform the policy update.
    \item Action space: When the action space is discrete, we divide by the expectation over the entire action space to compute $\mathbb{F} (\pi_k)(\cdot|s)$; When the the action space is continuous, we use 256 samples to estimate this expectation. 
\end{enumerate}

\textbf{Architectural details: } Apart from the policies for tabular domains (array) and continuous action spaces (generative actor critic network), all the other policies and values are constructed using the classic MLP networks with hidden layer sizes as listed in Table \ref{appendix_hyperparamaters_table}.

\newpage
\subsection*{ODRPO hyperparamaters and additional results: }

Our main experimental results are reported in Section \ref{experiment_section}. In addition, we provide the setting of hyperparamaters and network sizes of our \texttt{ODRPO} algorithm in Table \ref{appendix_hyperparamaters_table}. And we present the numerical results of the final performance comparison among our \texttt{ODRPO} algorithm and the baseline methods (i.e., TRPO, PPO, A2C) in Table \ref{appendix_performance_table}. Besides, the episode rewards during the training process for additional MuJuCo tasks are shown in Figure \ref{appendex_more_mujuco_results}.

\begin{figure}[H]
\centering
\begin{minipage}{.30\textwidth}
  \centering
  \includegraphics[width=0.9\linewidth]{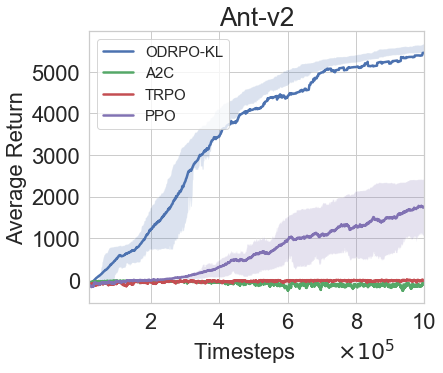}
\end{minipage}
\begin{minipage}{.30\textwidth}
  \centering
  \includegraphics[width=0.9\linewidth]{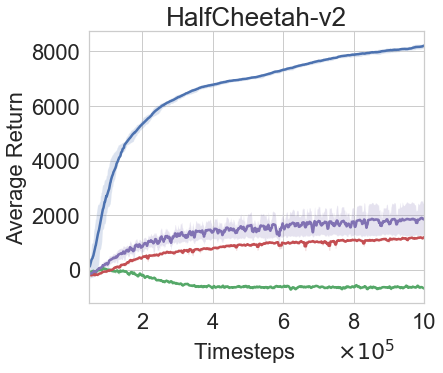}
\end{minipage}
\begin{minipage}{.30\textwidth}
  \centering
  \includegraphics[width=0.9\linewidth]{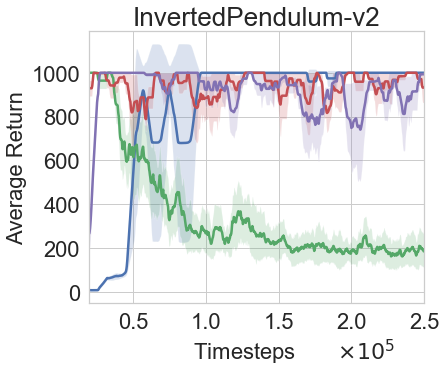}
\end{minipage}
\caption{Additional learning curves for MuJuCo tasks}
\label{appendex_more_mujuco_results}
\end{figure}

\begin{table}[H]
\centering
\footnotesize
\caption{\texttt{ODPRO} hyperparamaters and network sizes}
\renewcommand{\arraystretch}{1.25}
\begin{adjustbox}{width=\columnwidth,center}
\begin{tabular}{llllll}
 \hline
 & Taxi-v3, NChain-v0 & CartPole-v1 & Acrobot-v1 & LunarLander & Walker2d-v2, Hopper-v2 \\
 & CliffWalking-v0 & & & Continuous-v2 & Reacher-v2, HalfCheetah-v2 \\
 &&&&& Ant-v2, InvertedPendulum-v2\\
 \hline
 $\gamma$ & $0.9$ & $0.95$ & $0.95$ & $0.99$ & $0.99$ \\
 \hline
 $lr_{\pi}$ & \textbackslash & $10^{-2}$ & $5 \times 10^{-3}$ & $10^{-4}$ & $10^{-4}$ \\
 \hline
 $lr_{\text{value}}$ & $10^{-2}$ & $10^{-2}$ & $5 \times 10^{-3}$ & $10^{-4}$ & $10^{-3}$ \\
 \hline
 $|\mathcal{D}_k|$ & $60$ (Taxi); $1$ (Chain);  & 2 & 3 & \textbackslash & \textbackslash \\
 & $3$ (CliffWalking) &  &  & \\
 \hline
  $\pi$ size & 2D array & $[64,64]$ & $[64,64]$ & $[400,300]$ & $[400,300]$ \\
 \hline
 \text{Q/v} size & $[10,7,5]$ & $[64,64]$ & $[64,64]$ & $[400,300]$ & $[400,300]$ \\
 \hline
\end{tabular}
\label{appendix_hyperparamaters_table}
\end{adjustbox}
\end{table}

\begin{table}[H]
\caption{Averaged rewards over last 10\% episodes during the training process}
\renewcommand{\arraystretch}{1.5}
\centering
\begin{adjustbox}{width=\linewidth,center}
\begin{tabular}{llllll}
 \hline
Environment & \texttt{ODRPO-KL}  & \texttt{ODRPO-Wass} & TRPO & PPO & A2C \\
\hline
Taxi-v3 & $-16 \pm 10$ & $-75 \pm 17$ & $-201 \pm 1$ & $-343 \pm 36$ & $-304 \pm 27$\\ 
\hline
NChain-v0 & $3666 \pm 223$ & $3347 \pm 267$ & $3522 \pm 258$ & $3506 \pm 237$ & $1606 \pm 10$ \\
\hline
CliffWalking-v0 & $-40 \pm 23$ & $-35 \pm 15$ & $-169 \pm 36$ & $-3608 \pm 2172$ & $-5168 \pm 2247$ \\ 
\hline
CartPole-v1 & $401 \pm 42$ & $324 \pm 78$ & $307 \pm 68$ & $193 \pm 45$ & $267 \pm 61$ \\
\hline
Acrobot-v1 & $-97 \pm 4$ & $-153 \pm 7$ & $-143 \pm 13$ & $-85 \pm 5$ & $-452 \pm 28$\\
\hline
LunarLanderContinuous-v2 & $170 \pm 81$ & \hspace{0.9cm} \textbackslash & $-1 \pm 92$ & $93 \pm 85$ & $-262 \pm 130$\\
\hline
Walker2d-v2 & $3686 \pm 81$ & \hspace{0.9cm} \textbackslash & $2667 \pm 1023$ & $3179 \pm 747$ & $-4 \pm 2$\\
\hline
Hopper-v2 & $2781 \pm 537$ & \hspace{0.9cm} \textbackslash & $2713 \pm 424$ & $2590 \pm 849$ & $42 \pm 18$\\
\hline
Reacher-v2 & $-5.1 \pm 0.1$ & \hspace{0.9cm} \textbackslash & $-5.9 \pm 1.7$ & $-7.5 \pm 2.1$ & $-64.2 \pm 1.4$ \\
\hline
Ant-v2 & $5397 \pm 249$ & \hspace{0.9cm} \textbackslash & $-14 \pm 22$ & $1751 \pm 665$  & $-88 \pm 44$\\
\hline
HalfCheetah-v2 & $8174 \pm 92$ & \hspace{0.9cm} \textbackslash & $1178 \pm 47$ & $1844 \pm 590$ & $-684 \pm 46$ \\
\hline
InvertedPendulum-v2 & $1000 \pm 0$ & \hspace{0.9cm} \textbackslash & $953 \pm 47$ & $989 \pm 11$ & $197 \pm 80$ \\
\hline
\end{tabular}
\label{appendix_performance_table}
\end{adjustbox}
\end{table}

\end{document}